\pdfoutput=1
\documentclass[11pt, dvipsnames]{article}

\usepackage[final]{acl}

\usepackage{times}
\usepackage{latexsym}
\usepackage[T1]{fontenc}
\usepackage[utf8]{inputenc}
\usepackage{microtype}
\usepackage{inconsolata}

\usepackage{booktabs}    % for toprule, midrule, etc.
\usepackage{multirow}
\usepackage{tabularx}
\usepackage{colortbl}
\usepackage[inline]{enumitem}    % for spacing/structuring lists
\usepackage[fleqn]{amsmath}
\usepackage{centernot}
\usepackage{amssymb}
\usepackage{amsthm}      % for the theorem in appendix B
\usepackage{subcaption}  % for multiple figures each with its own subcaption PLUS an overall caption
\usepackage{rotating}    % for rotated text in tables
\usepackage[export]{adjustbox}
\usepackage{float}

\title{Paying Attention to Deflections: Mining Pragmatic Nuances for Whataboutism Detection in Online Discourse}

\author{Khiem Phi \quad Noushin Salek Faramarzi \quad Chenlu Wang \quad Ritwik Banerjee \\
        Department of Computer Science \\
        Stony Brook University, New York, USA \\
        \texttt{\small\{kphi, nsalekfarama, chenlwang, rbanerjee\}@cs.stonybrook.edu}}

\newcommand{\plus}{\raisebox{.1\height}{\scalebox{.8}{+}}}
\definecolor{bestblue}{RGB}{20,160,250}
\definecolor{worstyellow}{RGB}{240,200,100}
\newcommand{\sqboxs}{1.2ex}% the square size
\newcommand{\sqbox}[1]{\textcolor{#1}{\rule{\sqboxs}{\sqboxs}}}

\begin{document}
\maketitle

\begin{abstract}
Whataboutism, a potent tool for disrupting narratives and sowing distrust, remains under-explored in quantitative NLP research. Moreover, past work has not distinguished its use as a strategy for misinformation and propaganda from its use as a tool for pragmatic and semantic framing. We introduce new datasets from Twitter\footnote{We use the name ``Twitter'', since our data collection and analysis was conducted while the platform still used that name.} and YouTube, revealing overlaps as well as distinctions between whataboutism, propaganda, and the \textit{tu quoque} fallacy. Furthermore, drawing on recent work in linguistic semantics, we differentiate the `what about' lexical construct from whataboutism. Our experiments bring to light unique challenges in its accurate detection, prompting the introduction of a novel method using attention weights for negative sample mining. We report significant improvements of \textbf{4\%} and \textbf{10\%} over previous state-of-the-art methods in our Twitter and YouTube collections, respectively.\footnote{Code and data: \href{http://github.com/KhiemPhi/wabt-det}{github.com/KhiemPhi/wabt-det}.}

%Whataboutism is one of the most common and effective tools of disrupting narratives and sowing distrust yet it has received scant attention in quantitative natural language processing (NLP) research. Past work has not distinguished its use as a strategy for misinformation and propaganda from its use in pragmatic and semantic framing in discourse. We introduce two new datasets developed from Twitter\footnote{We use the name ``Twitter'', since our data collection and analysis was conducted while the platform still used that name.} and YouTube to demonstrate that instances of whataboutism in contemporary social media are often distinct from propaganda and the classical \textit{tu quoque} fallacy. Furthermore, based on recent work in linguistic semantics, we also show that the `what about' lexical construct and whataboutism are two separate discursive phenomena. Our technical experiments illustrate that there are several obstacles to accurate whataboutism detection (including its pragmatic variability and its distinction from the aforementioned lexical construct) by means of state-of-the-art transfer learning methods. To overcome these obstacles, we introduce a novel method to mine negative samples using attention weights instead of cosine distance as a measure of semantic similarity. We report significant improvements of \textbf{4\%} and \textbf{10\%} over recent quantitative methods on whataboutism detection in our Twitter and YouTube collections, respectively.\footnote{Code and data: \href{http://github.com/KhiemPhi/wabt-det}{http://github.com/KhiemPhi/wabt-det}.}
\end{abstract}

\section{Introduction}\label{sec:intro}
% --> introduce the problem
Whataboutism is the practice of deflecting criticism or avoiding an unfavorable issue by raising a different, more favorable matter, or by making a counter-accusation. Since its first use in 1974~\cite{zimmer2017roots}, it has emerged as a common variation of the classical fallacy known as \textit{tu quoque} (lit. ``you also'') -- attacking the opponent's behavior or action for being inconsistent with their argument, thereby discrediting them.
Despite significant work devoted to misinformation and propaganda, the detection of whataboutism has largely relied on the ease of tracking ``\textit{what about}'' phrases. This, however, is informed by popular notions of the phenomenon, leading to a na\"ive linguistic treatment and subsequently, a neglect of the unique challenges to its detection. Many ``what about'' phrases do not, in fact, signal propagandist use (Fig.~\ref{fig:comment-examples}).

\begin{figure}[!t]
\centering
\includegraphics[width=.94\linewidth, cfbox=gray 2pt 5pt]{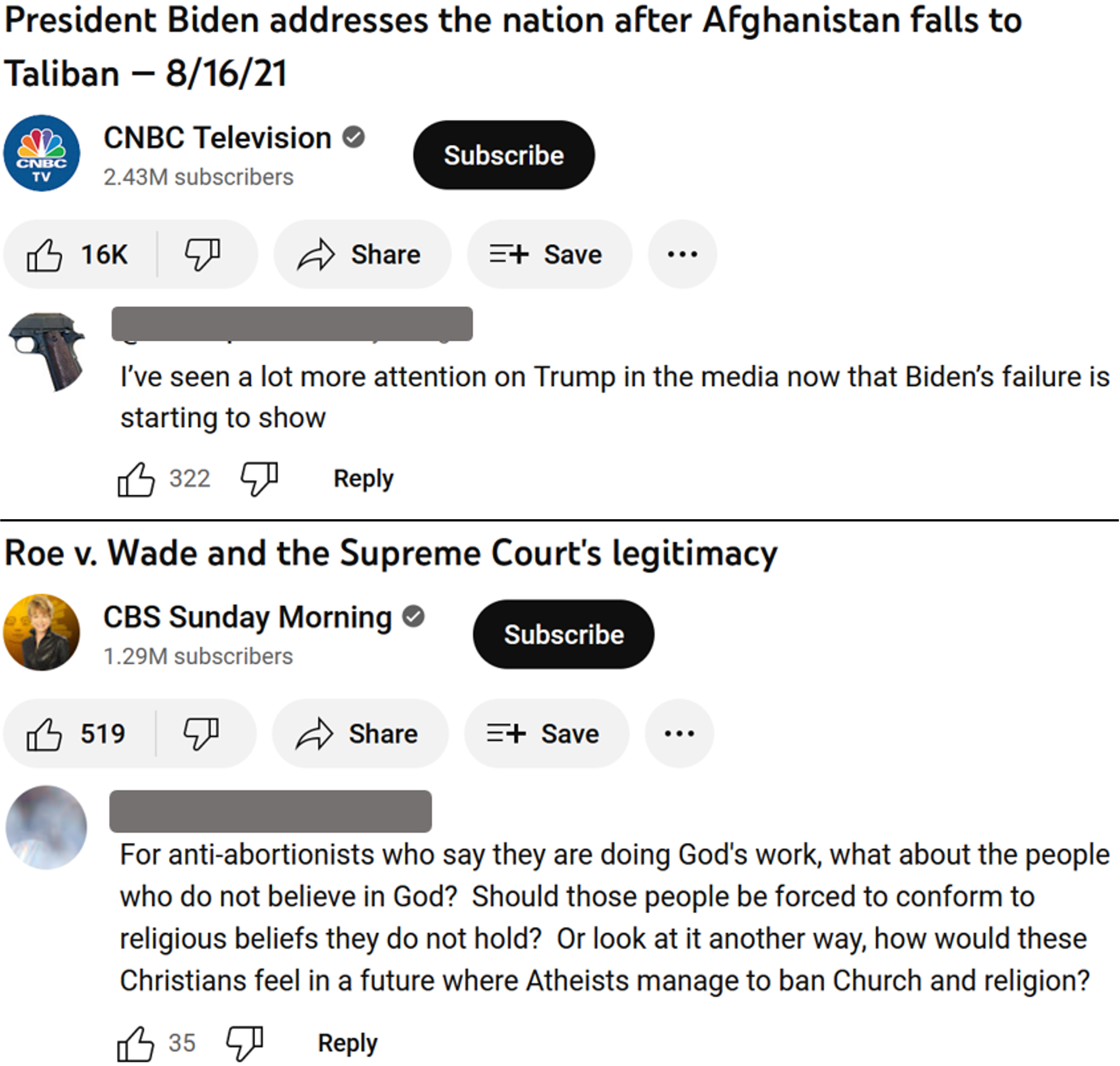}
\caption{\textit{What about} in YouTube comments: implicit use (top) to discredit the source and redirect the topic, and explicit use (bottom) as an attempt toward reasonable argumentation instead of propaganda.\vspace{-4pt}} %  (accessed: June 20, 2023)
\label{fig:comment-examples}
\end{figure}

\begin{figure*}[!t]
\centering
\includegraphics[width=.99\linewidth, cframe=gray 2pt]{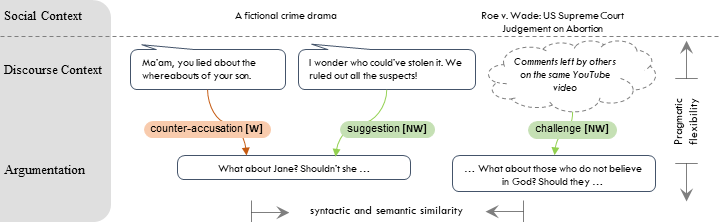}
\caption{Syntactically and semantically similar (or even identical) responses may exhibit extreme \textit{pragmatic flexibility}: being instances of whataboutism [\textsc{\textbf{\textcolor{ForestGreen}{w}}}] or not [\textsc{\textbf{\textcolor{BrickRed}{nw}}}] depending on the discursive context. Furthermore, the latter category may include valid argumentation tools such as suggestion or challenge, instead of propaganda.}
\label{fig:pragmatic-flexibility}
\end{figure*}
% --> show that the problem is important (and that it is unsolved)
Whataboutism drives information disorder\footnote{Following the Council of Europe report by \citet[pp.~10, 20]{wardle2017information}, `information disorder'' includes dis/misinformation (falsehoods with/out the intention to mislead), as well as malinformation (\textit{e.g.}, hate speech).} by derailing proper argumentation, instead of the relatively conspicuous act of misleading by directly lying. As epistemologist~\citet[p.~420]{fallis2015disinformation} argues, information disorder comprises two equally odious functions: (a) \textit{creating false beliefs}, and (b) \textit{preventing the creation of belief in truth}. Empirical efforts target the former through models for fact-checking and fake news detection~\cite{guo2022survey}. Whataboutism, on the other hand, is a prevalent and potent contributor to the latter as well. Given its deleterious effects on discourse and social cohesion, it is thus critical for scalable propaganda detection models to be informed by a deeper understanding of whataboutism so that valid argumentation is not conflated with propaganda. %This work provides just such a scrutiny, with the first rigorous framework for the analysis and detection of whataboutism, and a detailed discussion of the unique challenges therein.

This work provides just such a scrutiny: with two new datasets (\S\hspace{1pt}\ref{sec:datasets}), we describe the first rigorous framework for the analysis and detection of whataboutism, distinguishing valid argumentation from propagandist attempts to derail a narrative. We find that instances of whataboutism, redirection, and \textit{tu quoque} fallacies do not coincide in contemporary social media. Contrary to suppositions of prior work on propaganda detection, our analysis reveals that whataboutism serves purposes other than propaganda. Further, we demonstrate that a habitual application of transfer learning or contrastive learning are remarkably poor at detecting whataboutism.\footnote{Identical methods (see \S\hspace{1pt}\ref{sec:experiments}) have shown state-of-the-art results in various downstream tasks, notably computational propaganda detection. This clear disparity underscores the importance of recognizing whataboutism as overlapping with propaganda, yet distinct from it (as shown in Fig.~\ref{fig:pragmatic-flexibility}).}
To overcome this, we describe a novel method to mine negative samples using attention weights instead of contrast using cosine similarity (\S\hspace{1pt}\ref{sec:experiments}). With this new formulation of (dis)\,similarity achieving superior results (\S\hspace{1pt}\ref{sec:analysis}), our work holds implications for understanding complex discourse surrounding misinformation and propaganda, and modeling natural language pragmatics.

\section{Pragmatic flexibility}
%\label{sec:pragmatic-flexibility}
\paragraph{``What about'' $\centernot\Leftrightarrow$ whataboutism:}
We utilize the theory of \textit{structured meaning approach}~\cite{krifka2001structured} for a deeper examination beyond early datasets that appear to have conflated the two to a large extent~\cite{dasanmartino2020semeval}. Studies of the ``what about'' construct demonstrate extremely high semantic and pragmatic complexity, but has received little empirical attention~\cite{bledin2021topicality}. It can be used with a broad range of syntactic constituents, yielding diverse semantics. Further, it has often been framed in terms of an antecedent \textit{question under discussion} (QUD)~\cite{roberts2012information}, wherein a diversion is achieved by establishing the prejacent as a referential topic while simultaneously raising a new question about its properties~\cite{ebert2014unified}. Thus, it is not the bigram ``what about'', but the nature of the newly raised question, which determines whether whataboutism is being introduced in a discourse. To wit, Fig.~\ref{fig:comment-examples} depicts valid argumentation employing this phrase, while an instance of whataboutism does not.
%On the other hand, there are clear instances where ``what about'' is implicit -- to wit, an instance of whataboutism can be (and often is) paraphrased. Fig.~\ref{fig:comment-examples} illustrates that whataboutism can occur without the ``what about'' bigram, in stark contrast to valid argumentation employing it.

Whether ``what about'' is explicit or not, it can suggest a solution, challenge a statement, add specification to prior discussion, or laterally redirect the QUD~\cite{beaver2017questions, bledin2020resistance}. This extreme pragmatic flexibility makes it difficult to \textit{distinguish valid argumentation from propagandist use}. The problem is further compounded by the fact that the \textit{pragmatic distinctions coexist with lexical and semantic similarity}. Fig.~\ref{fig:pragmatic-flexibility} illustrates how the same text may serve as (i) harmful deflection via accusation or diversion, which are characteristics of propaganda, and (ii) accommodation or ``centering'' of attention.\footnote{The former characterize propaganda, while the latter is a valid act that, incidentally, has immense utility in NLP tasks like anaphora resolution~\cite{grosz1983providing, dekker1994predicate}.}
Thus, strategies relying solely on syntax or semantics may not differentiate these opposing discourse maneuvers.
%Thus, methods relying only on the syntactic or semantic properties of the statement may not distinguish between these opposing discourse maneuvers.

\section{The Datasets}\label{sec:datasets}
\citet{dasanmartino2019fine} curated a corpus for propaganda detection, containing 76 instances of whataboutism, which serves as a valuable foundation for the study of propaganda. However, its size and assumption (that every occurrence is a propagandist use) limit its suitability for our study. We introduce the TQ$^{\plus}$ collections,\footnote{To highlight that modern use of whataboutism surpasses the classical \textit{tu quoque} fallacy, our datasets are dubbed TQ$^{\plus}$, with platform-specific subscripts serving as mnemonic aids.} encompassing the social discourse around each instance. We hope this dataset, featuring ten times more labeled instances than the previous corpus, will spur deeper analyses of fallacies and inspire further research on the role of pragmatics in propaganda and misinformation.

\textbf{YouTube comments} have not received much attention in empirical studies. Our first collection comprises YouTube comments due to a balanced user distribution across demographics~\cite{statista2022gender, statista2022nations2, statista2022nations1, statista2022users} and a less regulated environment, fostering diverse opinions~\cite{mejova2012political}. Targeting socially divisive topics prone to whataboutism, we formulate search queries for six such topics on YouTube. From the top five most viewed videos per query, we collect English comments along with the title, transcript, the number of up-votes, and the publisher information. This corpus, TQ$^{\plus}_{\textsc{yt}}$, comprises 1,642 labeled comments from 17 videos across 6 topics, sorted by up-votes.

Similar to TQ$^{\plus}_{\textsc{yt}}$, our \textbf{Twitter dataset} TQ$^{\plus}_{\textsc{tw}}$ focuses on socially relevant topics that tend to elicit strong emotional responses. We gather English tweets and their replies 8 such topics. The original tweet offers the pragmatic context for identifying whataboutism in responses. While six topics overlap with TQ$^{\plus}_{\textsc{yt}}$, the other two contain fewer tweets, so that our corpus can be used to analyze model performance on topics with limited data.
For each topic, we gather tweets with significant engagement, filtering out threads with less than 200 messages. Within each thread, we exclude messages that do not directly respond to the original tweet, lack opinions or discursive content (\textit{e.g.}, ``wow''), are socially too inappropriate, consist solely of emojis/emotions, or contain images/videos. In TQ$^{\plus}_{\textsc{tw}}$, each datum comprises a tweet-reply pair, and the collection consists of 1,202 messages.
%For each topic, we use the Twitter API to collect tweets based on specific search queries deemed relevant and characteristic of that topic.
%For example, tweets about the Russo-Ukrainian war are collected using the search query ``Zelensky''.
% For each topic, we collect tweets with high user engagement by retaining only those that have at least 200 messages in their conversation thread. From each thread, we remove messages that are: 
% \begin{itemize*}
% \item[a)] not responses to the original tweet;\footnote{Without this filter, a correct use of Twitter threads for pragmatic understanding would require accurate multi-party conversation disentanglement, a formidable research task on its own~\cite{gu2022who, ganesh2023survey}.}
% \item[b)] lacking substantive opinions or have little discursive value (\textit{e.g.}, ``Oh I see'' or ``wow'');
% \item[c)] too inappropriate/insensitive toward any ethnicity;
% \item[d)] composed entirely of emojis and/or emoticons; and
% \item[e)] images or videos.
% \end{itemize*}
% Finally, TQ$^{\plus}_{\textsc{tw}}$ comprises 1,202 messages across 8 topics.

We use stratified partitioning to divide both datasets into training, validation, and testing sections, ensuring an even distribution of comments across topics in each. Within each topic, comments were split into 80\% training, 5\% validation, and 15\% testing, maintaining the class distribution throughout. These stratified segments are then combined across topics to form the final training, validation, and test sets. The distribution of labeled data across these sets, as well as the class-wise distributions, are described in Appendix~\ref{appendix:datasheet}.

%We partition TQ$^{\plus}_{\textsc{yt}}$ into training, validation, and testing sections, which are further partitioned by topics such that we obtain an even distribution of comments for each topic across each section. The comments are split within each topic into 80\% training, 5\% validation, and 15\% testing. We then combine the topic-specific training portions to form the training set (and similarly, the validation and test sets) for the entire dataset. Our Twitter data, where each datum comprises a tweet-reply pair, is partitioned similarly.

%Our Twitter data is partitioned similarly: 80\%, 5\%, and 15\% for training, validation, and testing, respectively. Each partition retains the topic-ratios of the undivided collection. Unlike the YouTube collection, where each datum is a single comment found under a video, an instance in TQ$^{\plus}_{\textsc{tw}}$ consists of a tweet-reply pair.

%\citet{shoukri2003measures} suggests no benefits from more than three annotators for a binary variable.
We enlist three annotators with native fluency in English who independently assign binary labels to each comment.\footnote{\citet{shoukri2003measures} suggests no additional benefits from more than three annotators for a dichotomous variable.} To 
attain an understanding of the broader sociopolitical discourse pertaining to each topic, annotators dedicate considerable time to carefully review the entire YouTube video or read the complete Twitter conversation before labeling each instance. This meticulous process is expected to ensure diligent annotation and achieve a high level of data fidelity, despite limiting the size of the datasets. The final label is determined by majority vote, with inter-annotator agreements measured by Fleiss' kappa~\cite{fleiss1973} at $\kappa = 0.65$ (TQ$^{\plus}_{\textsc{yt}}$) and $\kappa = 0.75$ (TQ$^{\plus}_{\textsc{tw}}$).\footnote{There is no single threshold for a good value of $\kappa$. The widely used interpretations introduced by \citet{landis1977} regard these scores as ``substantial agreement''.}

% \paragraph{Data annotation:}
% According to \citet{shoukri2003measures}, more than three annotators on a dichotomous variable does not provide additional advantages. Therefore, we employ three annotators with native fluency in English, who work independently to provide a binary label for each individual comment. We require each annotator to watch the YouTube video before labeling the comments made on that video to ensure the appropriate contextual knowledge for annotation. Similarly, each annotator is required to read the original tweet before labeling a response. The final ground-truth label is decided by majority vote. We report inter-annotator agreements of $\kappa_{\textrm{Fleiss}} = 0.65$ and $\kappa_{\textrm{Fleiss}} = 0.75$ on the YouTube and Twitter collections, respectively.

Following recent recommendations \cite{gebru2018datasheets, bender2018data}, we include two comprehensive transparency artifacts:
\begin{itemize*}
\item[(a)] the data statement (Appendix~\ref{appendix:datasheet}) and
\item[(b)] the annotation guide/codebook, included in our data repository.
\end{itemize*}

\begingroup
\renewcommand{\arraystretch}{0.75}
\setlength\tabcolsep{6pt}
\begin{table*}[!t]
\centering
{\footnotesize
\begin{tabularx}{\linewidth}{@{}p{.34\linewidth} X c c c@{}}
\toprule
\multicolumn{1}{c}{\textbf{Original context}} & \multicolumn{1}{c}{\textbf{Comment}} & \textbf{\textsc{w}} & \textbf{$\boldsymbol{\Rightarrow}$} & \textbf{\textsc{tq}} \\
\midrule
\multirow[t]{2}{=}{YT Video: One dead after `Unite the Right' rally in Virginia (Fox News)} & (1) if such things happened in china and russia, what would cnn and nytimes say? & \checkmark & \checkmark & \smallskip \\
  & (2) If the perp was an Islamist, imagine how the left would dismiss what he did as not being indicative of Islam (like they always do). & \checkmark & \checkmark & \smallskip \\
\noalign{\vskip 0.8ex}
% \hdashline\noalign{\vskip 0.8ex}
 YT Video: Outrage grows over Russian bounties (ABC News) & (3) They did the same thing to indigenous folks for bounty hunters and military, the US military deserves this. & & & \checkmark \smallskip \\
\noalign{\vskip 0.8ex}
% \hdashline\noalign{\vskip 0.8ex}
 YT Video: Russians Flee Into Exile Because Of Putin's War With Ukraine: NYT (MSNBC) & (4) Fox News, Republicans and Trump all defended this monster. They are now hoping and praying that you forget that part. & & \checkmark & \smallskip \\
\noalign{\vskip 0.8ex}
% \hdashline\noalign{\vskip 0.8ex}
 \multirow[t]{2}{=}{YT Video: Russia is aware it is not winning the war in Ukraine: Brookings senior fellow (CNBC)} & (5) Russia bombing buildings four miles from Kyiv's center doesn't sound like losing. How close have the Ukrainians gotten to Moscow? & \checkmark & & \checkmark \smallskip \\
   & (6) USA baby not didn't war crime afganistan Iraq & \checkmark & \checkmark & \checkmark \smallskip \\
\noalign{\vskip 0.8ex}
% \hdashline\noalign{\vskip 0.8ex}
 \multirow[t]{3}{=}{Tweet: We united the world to protect Ukraine, we will unite the world to restore justice. Russian invaders will be legally and fairly held to account for all war crimes. The terrorist state will be held to account for the crime of aggression. ($@$ZelenskyyUa; Mar 3, 2023)}
   & (7) Coming from a govt that congratulated Bola Tinubu of Nigeria. Funny people & \checkmark & \checkmark &  \smallskip \\
   & (8) And what of Ukraine's war crimes? Torture and murder of POWs, firing from civilian positions, using human shields or shooting civilians for receiving Russian ration packs. ALL war crimes need to be prosecuted, not just those by Russians. & \checkmark & \checkmark & \checkmark \medskip\\
\bottomrule
\end{tabularx}
}
\caption{Comments and their corresponding contexts (YouTube video title, or tweet to which they are responding) along with three facets: whataboutism (\textsc{\textbf{w}}), topic redirection ($\boldsymbol{\Rightarrow}$), and conformity with classical \textit{tu quoque} (\textsc{\textbf{tq}}).}
\label{tab:thematic-examples}
\end{table*}
\endgroup

\subsection{A thematic scrutiny}\label{ssec:thematic-scrutiny}
Inspired by the wide adoption of thematic analysis in qualitative research~\cite{braun2006thematic, guest2011thematic}, we offer a qualitative discussion to uncover recurring themes in our data. It stands not as an extension, but in complement, to the quantitative study based on corpus annotations. While rare in computational research, thematic analysis has been closely linked to foundational work on word senses and content analysis~\cite{stone1966general, litkowski1997desiderata, ide1998introduction}. %As such, we include this method due to the qualitative dimensions of our topic.

\paragraph{Discourse coherence:}
Comment-threads on YouTube videos lack coherent discourse. A comment often serves as a standalone response to the video or earlier comments. This differs from discourse structures in articles, interviews, or debates, where whataboutism has been observed~\cite{putz2016whataboutism, dykstra2020rhetoric}. While Twitter threads also show some deviation, our filtering ensures that TQ$^{\plus}_{\textsc{tw}}$ offers a cleaner statement-response format.

%Since the user is not responding to any question or criticism, whataboutism is rarely a defensive maneuver to avoid an unfavorable issue. Instead, it becomes an instrument of precluding the relevant issue by \textit{introducing} an accusation (instead of raising a \textit{counter}-accusation). Very often, users employ whataboutism to facilitate and influence the  perception, comprehension, or evaluation of an issue through the \textbf{lens of domestic political binaries}. This can be seen in Table~\ref{tab:thematic-examples}: on a video by the conservative American media outlet\footnote{The liberal, conservative, or centrist leanings mentioned in this discussion are obtained from \citet{adfontes}.} Fox News, (1) reveals the author's belief about the hypocrisy of two liberal media outlets, CNN and the New York Times. A similar tone is taken by (2), mentioning ``the left''. Similarly, opposition to conservative entities are often found under content produced by liberal organizations (4). Opposition to presented narratives are seen even on comparatively centrist channels (3).

\paragraph{Perceptions, framing, and \textit{Tu Quoque}:}
We see that users seldom employ whataboutism as a defensive tactic to deflect criticism. Instead, it is frequently wielded to preclude the relevant issue by \textit{introducing} an accusation, diverging from the traditional view of whataboutism as a response with a  \textit{counter}-accusation. Users frequently employ whataboutism to shape perceptions, understanding, or evaluation of an issue, particularly through the \textbf{lens of domestic political divides}. Table~\ref{tab:thematic-examples} illustrates this trend: comments on a Fox News video critique liberal media outlets like CNN and the New York Times (1, 2), while criticisms of conservative entities appear under content from liberal organizations (4), and even centrist channels (3).\footnote{The liberal, conservative, or centrist leanings mentioned in this discussion are obtained from \citet{adfontes}.}

% Pointing out \textbf{selection bias} or \textbf{hypocrisy} (alleged or otherwise) is another common theme. Users often underscore issues they find the news to be ignoring, and frame it in a new context using whataboutism. To extent to which such framing can be called propaganda is difficult to judge. After all, this style of discourse is widespread in journalism, and is often utilized in article headlines from news organizations held in high regard.\footnote{E.g.: (1) Dhanesha, N. \textit{Climate fixes are all aimed at property owners. What about renters?} July 27, 2022. Vox. \url{www.vox.com/the-highlight/23198145/renters-climate-change-solutions}. Accessed: June 23, 2023. (2) Hawkins, A. and Davidson, H. \textit{As the west tries to limit TikTok's reach, what about China’s other apps?} April 12, 2023. The Guardian. \url{www.theguardian.com/technology/2023/apr/12/tiktok-china-apps-national-security-wechat-shein}. Accessed: Jan 3, 2024.} Social media users follow suit: in (5), the report's central claim is countered by raising what the author perceives as disregarded evidence, \textit{viz.}, ``Russia bombing four miles from Kyiv's center''.\footnote{This comment implies that Ukraine moving closer to Moscow and Russia moving closer to Kyiv are the only possible outcomes, showcasing two fallacies: whataboutism and false dichotomy.}
% In (7), we see an appeal to hypocrisy by means of whataboutism, where it is a deflection using counter-accusation, but distinct from the classical \textit{tu quoque} fallacy. %We also note the use of whataboutism to view an event through a \textbf{lens of justice or retribution} -- possibly with a sense of schadenfreude (3).

Users commonly employ whataboutism to highlight perceived \textbf{selection bias} or \textbf{hypocrisy} in news coverage, reframing overlooked issues in a new context. Assessing whether such reframing constitutes propaganda is complex, given its prevalence even in journalism from reputable organizations.\footnote{For example: (1) Dhanesha, N. \href{http://www.vox.com/the-highlight/23198145/renters-climate-change-solutions}{\textit{Climate fixes are all aimed at property owners. What about renters?}} July 27, 2022. Vox. Accessed: 06.23.2023. (2) Hawkins, A. and Davidson, H. \href{http://www.theguardian.com/technology/2023/apr/12/tiktok-china-apps-national-security-wechat-shein}{\textit{As the west tries to limit TikTok's reach, what about China’s other apps?}} April 12, 2023. The Guardian. Accessed: 01.03.2024.}
Social media users similarly engage in this practice: in (5), a central claim is challenged with what the author views as overlooked evidence;\footnote{Table~\ref{tab:thematic-examples}:(5) implies there are the only possible outcomes, showcasing the false dichotomy fallacy.} and (7) demonstrates the use of whataboutism in an appeal to hypocrisy but the deflection is not a counter-accusation, diverging from the \textit{tu quoque} fallacy.

% There is much in common between
% \begin{itemize*}
% \item[(a)] whataboutism as a propaganda tool,
% \item[(b)] whataboutism in general discourse, and
% \item[(c)] the ``what about'' bigram construct.
% \end{itemize*}
% But our thematic analysis shows that there are enough distinctions such that \textit{none of these three can imply another}, even in the presence of topical redirection. We frequently observe comments like (4) in Table~\ref{tab:thematic-examples}, which indulge in overt redirection without whataboutism. Others like (5), however, stop short of a complete redirection, instead re-centering the attention from a different perspective on the same topic.
There are similarities between whataboutism in propaganda, whataboutism in general discourse, and the ``what about'' bigram. Our thematic analysis, however, reveals distinct characteristics, preventing one from implying the other, despite topical redirection. Some comments, like Table~\ref{tab:thematic-examples}:(4), overtly redirect without using whataboutism, while others, like (5), attempt to reframe the same topic.
% We observe contemporary whataboutism to be distinct from the classical \textit{\textbf{tu quoque}} fallacy. The Latin phrase translates to ``you also'', but we see many instances of whataboutism redirecting to an entity not directly involved or represented in the video or the comment thread. In lieu of a traditional two-person conversation, one may argue that the author of a comment views the video or the authors of previous comments as a representative of ``tu''. For instance, in examples (3) and (6) in Table~\ref{tab:thematic-examples}, the authors may have perceived the videos as representing the American government, and thus voice their opinions regarding past military actions. Still, we frequently see instances like (2), which introduces an accusation about ``the left'', an entity absent from the discourse. By comparison, Twitter provides a collection where instances of whataboutism are relatively more consistent with the classical notions of ``tu quoque'', as seen in (8). This is unsurprising, given the more direct tweet-response discourse structure of a Twitter thread.
Furthermore, contemporary whataboutism differs from the classical fallacy as they deflect to entities not directly involved in the conversation. Rather than a two-person interaction, commenters may perceive the video or previous posts as representative of ``tu''. For instance, in Table~\ref{tab:thematic-examples}:(3, 6), they may view the videos as representing the U.S. government and comment on past military actions. However, instances like (2) introduce accusations unrelated to the discourse (``the left''). Twitter examples, like (8), align more closely with classical \textit{tu quoque}, likely due to the direct tweet-response structure.

\begin{figure*}[!t]
\centering
\includegraphics[width=.99\textwidth, cframe=gray 2pt]{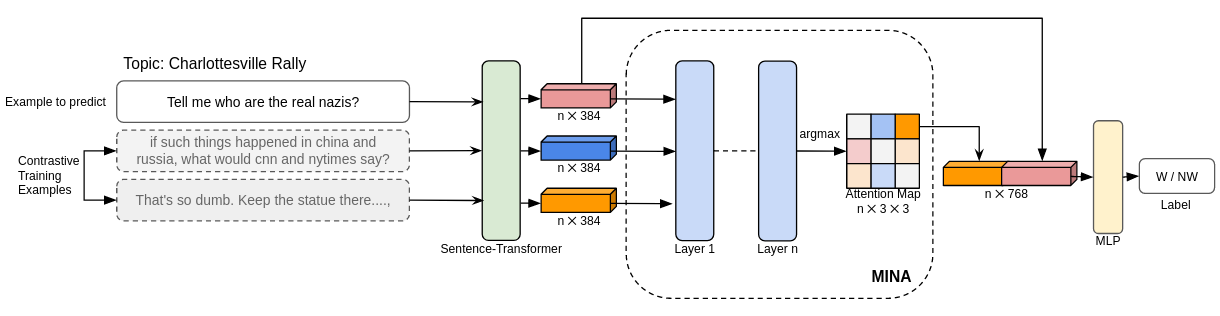}
\caption{MINA (\textbf{Mi}ning \textbf{N}egatives with \textbf{A}ttention) employs attention weights in the final layer of the Transformer encoder as a measure of pragmatic contrast. The complete architecture is shown here \textit{in situ}.}
\label{fig:mina-architecture}
\end{figure*}
\paragraph{Sarcasm and irony:}
Since subjectivity and opinion influence people significantly~\cite{picard2000affective}, we analyze our datasets for subjective expressions and emotive content. We find that explicit sentiment polarity or emotions are rare, while sarcasm or irony abound, as seen in Table~\ref{tab:thematic-examples}:(1, 6, 7).

\section{Experiments}
\label{sec:experiments}
We start with two sets of experiments before describing our novel approach: (i) transfer-learning techniques akin to recent empirical research on propaganda detection, and (ii) possible enhancements to these baselines by incorporating conventional measures of semantic (dis-)\,similarity.

\subsection{Baselines: (Habitual) Transfer learning}
To establish competitive baselines, we look to the best performing models from the SemEval-2020 propaganda detection task~\cite{dasanmartino2020semeval}.\footnote{While \citet{piskorski2023semeval} present a related task, we opt not to use the top models of that task due to: (a) the multilingual corpus; (b) their use of fine-tuned Transformer models (or their ensembles), much like the models in the earlier task; and (c) the poorer English-language whataboutism detection results, compared to the SemEval-2020 task.}
We incorporate a fully connected layer and fine-tune them using the following configurations on a single NVIDIA Titan XP GPU: (i) batch size of 80 for BERT-base and 40 for RoBERTa, (ii) 10 epochs, (iii) maximum sequence length of 256, and (iv) learning rate ($\eta$) set to $1 \times 10^{-4}$.
%We add a fully connected layer and fine-tune them with these configurations on a single NVIDIA Titan XP GPU: (i) batch size: 80 for BERT-base and 40 for RoBERTa, (ii) 10 epochs, (iii) maximum sequence length: 256, and (iv) $\eta$: $1 \times 10^{-4}$.
%\footnote{The RoBERTa representations have twice the dimensionality of BERT embeddings.}

Further, along the lines adopted by \citet{vlad2019sentence} and \citet{yu2021interpretable}, we include another baseline that integrates affective language with task-specific fine-tuning. Due to our corpus being rich in irony and sarcasm but sparse in explicit expressions of sentiment, we fine-tune a RoBERTa model pretrained on irony detection \cite{barbieri2020tweeteval} with the same training configuration.

\subsection{(Conventional) Semantic similarity}
\label{ssec:experiments-semantic-similarity}
Given the diversionary nature of whataboutism, identifying it hinges on discerning a sentence's relation to the antecedent question under discussion (QUD). For improvements over baseline results, we thus explore models that promise a deeper semantic understanding within a topic's discourse. Given the inadequate performance -- particularly on YouTube comments -- of language models pretrained on token prediction, we adopt SBERT \cite{reimers2019sentence}, which is trained on natural language inference (NLI) datasets \cite{bowman2015large, williams2018broad}, and excels in NLI and semantic textual similarity (STS) benchmarks. %to transfer its \khiemedit{understanding of} semantic similarity into topical similarity.
We fine-tune its embeddings with a configuration identical to that described earlier for BERT.

We also explore two additional contrastive learning methods: (i) ``Correct and Smooth'' (C\&S) \cite{huang2021combining}, which constructs a graph where the embeddings serve as nodes and cosine similarity provides edge weights, and tunes the embeddings via label propagation~\cite{zhou2003learning}; and (ii) SimCSE \cite{gao2021simcse}, which has achieved notable results in STS tasks with supervision from annotated pairs in NLI datasets (with \textit{entailment} and \textit{contradiction} pairs serving as positive and negative samples, respectively).

%\khiemedit{Despite this intuition, these methods still struggle on our dataset as verified empirically (\S\,\ref{sec:analysis}). A recurring theme among these methodologies is their frequent utilization of cosine similarity within the embedding space to separate positive and negative samples. However, the cosine similarity metric may be inadequate for high-dimensional vector spaces, as it becomes highly probable that any two vectors are nearly orthogonal. Furthermore, due to the pragmatic flexibility present within our dataset, we conjecture that current approaches relying on cosine similarity may struggle to detect whataboutism. This challenge spurs our exploration beyond traditional semantic similarity to capture pragmatic differences in contemporary social media discourse.}

\begingroup
\renewcommand{\arraystretch}{1.25}
\setlength{\tabcolsep}{3pt}
\begin{table*}[!t]
\begin{minipage}{.43\textwidth}
\small
\centering
\begin{tabularx}{\linewidth}{@{}>{\raggedright\arraybackslash}X c >{\centering\arraybackslash}X >{\centering\arraybackslash}X >{\centering\arraybackslash}X >{\centering\arraybackslash}X >{\centering\arraybackslash}X >{\centering\arraybackslash}X >{\raggedleft\arraybackslash}X}
\multicolumn{8}{c}{(a) Whataboutism detection on YouTube comments}\\
\midrule
Model & \phantom{xxxxxxxx} & \multicolumn{3}{c}{\textbf{W}} & \multicolumn{3}{c}{\textbf{NW}} \\
& & \textbf{P} & \textbf{R} & \textbf{F}\textsubscript{1} & \textbf{P} & \textbf{R} & \textbf{\textit{F}}\textsubscript{1}\\
\midrule
\multicolumn{8}{@{}l}{\textit{Transfer learning (baseline models)}:}\\
\textsc{bert}
  & & \cellcolor[RGB]{20,160,250}{\textcolor{white}{0.72}} & \cellcolor[RGB]{137,181,170}{0.36} & \cellcolor[RGB]{63,168,221}{0.48} & 0.92 & 0.98 & 0.95\\
\textsc{r}o\textsc{bert}a
  & & \cellcolor[RGB]{203,193,125}{0.23} & \cellcolor[RGB]{20,160,250}{\textcolor{white}{0.71}} & \cellcolor[RGB]{124,179,179}{0.34} & 0.94 & 0.64 & 0.76\\
\textsc{r}o\textsc{bert}a\textsubscript{\,Irony}
  & & \cellcolor[RGB]{177,188,143}{0.30} & \cellcolor[RGB]{157,185,157}{0.30} & \cellcolor[RGB]{141,182,168}{0.30} & 0.93 & 0.99 & 0.96\\
\multicolumn{8}{@{}l}{\textit{Based on conventional measures of semantic similarity}:}\\
\textsc{sbert}
  & & \cellcolor[RGB]{240,200,100}{0.13} & \cellcolor[RGB]{240,200,100}{0.05} & \cellcolor[RGB]{240,200,100}{0.07} & 0.96 & 0.99 & 0.97\\
\textsc{c\&s}
  & & \cellcolor[RGB]{121,178,181}{0.45} & \cellcolor[RGB]{107,176,191}{0.45} & \cellcolor[RGB]{76,170,212}{0.45} & 0.91 & 0.91 & 0.91\\
\textsc{s}im\textsc{cse}
  & & \cellcolor[RGB]{188,191,136}{0.27} & \cellcolor[RGB]{30,162,243}{0.68} & \cellcolor[RGB]{106,176,191}{0.38} & 0.94 & 0.72 & 0.82\\
\multicolumn{8}{@{}l}{\textit{Mining negatives with attention} (\textsc{mina}):}\\
\textsc{sbert}
  & & \cellcolor[RGB]{54,166,227}{0.63} & \cellcolor[RGB]{80,171,209}{0.53} & \cellcolor[RGB]{20,160,250}{\textcolor{white}{0.58$^*$}} & 0.94 & 0.96 & 0.95\\
\textsc{r}o\textsc{bert}a\textsubscript{\,Irony}
  & & \cellcolor[RGB]{76,170,212}{0.57} & \cellcolor[RGB]{100,175,195}{0.47} & \cellcolor[RGB]{46,165,232}{0.52} & 0.94 & 0.95 & 0.94\\
\toprule
\end{tabularx}
\end{minipage}\hfill%
\begin{minipage}{.5\textwidth}
\small
\centering
\begin{tabularx}{\linewidth}{@{}>{\raggedright\arraybackslash}X c >{\centering\arraybackslash}X >{\centering\arraybackslash}X >{\centering\arraybackslash}X >{\centering\arraybackslash}X >{\centering\arraybackslash}X >{\centering\arraybackslash}X >{\raggedleft\arraybackslash}X}
\multicolumn{9}{c}{(b) Whataboutism detection on Twitter replies}\\
\midrule
Model & \phantom{xxxxxxxx} & \multicolumn{3}{c}{\textbf{W}} & \multicolumn{3}{c}{\textbf{NW}} & \multicolumn{1}{c}{Params}\\
& & \textbf{P} & \textbf{R} & \textbf{F}\textsubscript{1} & \textbf{P} & \textbf{R} & \textbf{\textit{F}}\textsubscript{1} & ($M$)\\
\midrule
\multicolumn{9}{@{}l}{\textit{Transfer learning (baseline models)}:}\\
\textsc{bert}
  & & \cellcolor[RGB]{181,189,141}{0.64} & \cellcolor[RGB]{109,176,189}{0.76} & \cellcolor[RGB]{103,175,194}{0.69} & 0.85 & 0.76 & 0.80 & 110\\
\textsc{r}o\textsc{bert}a
  & & \cellcolor[RGB]{181,189,141}{0.64} & \cellcolor[RGB]{85,172,205}{0.80} & \cellcolor[RGB]{103,175,194}{0.71} & 0.87 & 0.75 & 0.81 & 124\\
\textsc{r}o\textsc{bert}a\textsubscript{\,Irony}
  & & \cellcolor[RGB]{240,200,100}{0.54} & \cellcolor[RGB]{32,162,242}{\textcolor{white}{0.89}} & \cellcolor[RGB]{130,180,175}{0.67} & 0.91 & 0.57 & 0.70 & 124\\
\multicolumn{9}{@{}l}{\textit{Based on conventional measures of semantic similarity}:}\\
\textsc{sbert}
  & & \cellcolor[RGB]{181,189,141}{0.64} & \cellcolor[RGB]{109,176,189}{0.76} & \cellcolor[RGB]{103,175,194}{0.69} & 0.85 & 0.76 & 0.80 & 22.7\\
\textsc{c\&s}
  & & \cellcolor[RGB]{222,197,112}{0.57} & \cellcolor[RGB]{204,194,124}{0.60} & \cellcolor[RGB]{240,200,100}{0.59} & 0.95 & 0.91 & 0.93 & 22.7\\
\textsc{s}im\textsc{cse}
  & & \cellcolor[RGB]{181,189,141}{0.64} & \cellcolor[RGB]{85,172,205}{0.80} & \cellcolor[RGB]{89,173,203}{0.71} & 0.87 & 0.74 & 0.80 & 22.7\\
\multicolumn{9}{@{}l}{\textit{Mining negatives with attention} (\textsc{mina}):}\\
\textsc{sbert}
  & & \cellcolor[RGB]{127,179,177}{0.73} & \cellcolor[RGB]{97,174,197}{0.78} & \cellcolor[RGB]{20,160,250}{\textcolor{white}{0.75$^*$}} & 0.87 & 0.84 & 0.85 & 31\\
\textsc{r}o\textsc{bert}a\textsubscript{\,Irony}
  & & \cellcolor[RGB]{20,160,250}{\textcolor{white}{0.91}} & \cellcolor[RGB]{222,197,112}{0.57} & \cellcolor[RGB]{89,173,203}{0.70} & 0.54 & 0.89 & 0.67 & 133\\
\toprule
\end{tabularx}
\end{minipage}
\caption{Whataboutism detection results on (a) YouTube and (b) Twitter: baseline results (top) using transfer-learning with fine-tuning; (middle) conventional semantic similarity measures; and (bottom) dis/similarity based on \textbf{\textsc{m}}ining \textbf{\textsc{n}}egative samples with \textbf{\textsc{a}}ttention (\textsc{mina}). Macro-average (\textbf{P})recision, (\textbf{R})ecall, and \textbf{\textit{F}}\textsubscript{1} for the target class are shown on column-wise color gradients, with blue (\sqbox{bestblue}) indicating the best performance and yellow (\sqbox{worstyellow}) the worst.}
\label{tab:results}
\end{table*}
\endgroup
Notwithstanding their proficiency in STS benchmarks, the results of our experiments (\S\hspace{1pt}\ref{sec:analysis}) reveal these models to be inadequate in understanding the pragmatic variations prevalent within each topic's social discourse. A common thread in their methodology is the use of cosine similarity to differentiate between instances from positive and negative samples. Cosine similarity may present challenges in high-dimensional vector spaces, as it becomes highly probable that any two vectors are nearly orthogonal (Appendix~\ref{app:proof} offers a formal proof). Further, as experiments reveal significant room for improvement over these models, we conjecture that approaches relying on cosine measure may struggle to detect whataboutism. This challenge drives us to explore beyond semantic similarity to capture pragmatic differences in modern social media discourse, with empirical results described in \S\hspace{1pt}\ref{sec:analysis}.

\subsection{Mining negatives with attention}
Given the centrality of topical redirection in whataboutism, any representation distinguishing it requires an understanding of the pragmatic context, namely the antecedent QUD. We thus model each comment by training on tuples comprising the comment $t$ (which we seek to classify), along with $c$ comments each from the \textsc{\textcolor{ForestGreen}{w}} and \textsc{\textcolor{BrickRed}{nw}} classes. For our collection of YouTube comments (TQ$^{\plus}_{\textsc{yt}}$), the 2$c$ comments are selected from the same video as $t$; and for TQ$^{\plus}_{\textsc{tw}}$, from the same thread as the original tweet. This hyperparameter $c$ controls context incorporation during training. We reason that such tuples can better capture semantic redirection \textit{within the ambient context of the social topic}, than embeddings based solely on global distributional semantics~\cite{lenci2022}. With this setup\footnote{The language model $m$, which provides the representation of $t$ and contextual comments, can be viewed as a discrete parameter of the complete approach, where one model may be swapped out for another.}, the most dissimilar examples can be identified using a context tuple. Then, combining them with a representation of $t$ will encode pragmatic shifts through an implicit \textit{grounding of the text in the surrounding social discourse}.\footnote{In effect, we surmise that the pragmatic act is grounded not just in the context of immediately surrounding texts, but in a sampling of the entire social discourse about that topic. This is similar in spirit to visually grounding objects through cross-modal attention \cite{ilinykh2022attention}.}

Specifically, we apply a Transformer encoder with $d$ layers and $h$ attention heads \cite{vaswani2017attention} to incorporate cross-attention between the elements of the context tuple. We then extract the cross-attention map from the encoder's final layer for a similarity matrix. After identifying the example with the highest numeric value, we concatenate its text with $t$. This combined embedding is fed into a multi-layer perceptron for final classification. Termed \textbf{\textit{mi}}ning \textbf{\textit{n}}egatives with \textbf{\textit{a}}ttention (\textsc{mina}), this method employs cross-attention scores to mine the most pragmatically dissimilar examples in the discourse surrounding $t$. Fig.~\ref{fig:mina-architecture} is an \textit{in situ} illustration of its complete architecture. Ablation studies (\S\ref{sec:analysis}) dictate our chosen configuration of 2 encoder layers, 32 attention heads, SBERT embeddings as input, and a context size of 1.

% "We extract the cross-attention map from the final encoder layer for similarity computation."

%\khiemedit{to perform cross-attention between the elements of the context tuple}. We then extract the cross attention map from the final layer of the encoder and utilize it as a similarity matrix. After identifying the example with the highest numeric value, we concatenate its text with $t$. This combined embedding is fed into a multi-layer perceptron for final classification. Termed \textbf{M}ining \textbf{N}egatives with \textbf{A}attention, this method employs cross-attention scores \khiemedit{along a tuple-based topical representation} to mine the most dissimilar examples in the discourse surrounding $t$. Fig.~\ref{fig:mina-architecture} is an \textit{in situ} illustration of the complete architecture of MINA. In our experiments, we use a configuration with 2 encoder layers, 32 attention heads, SBERT embeddings as input, and a context size of 1 for the tuples. This configuration is chosen based on a series of ablation studies, presented after the following discussion of the primary results.

\section{Results and analysis}\label{sec:analysis}
Table~\ref{tab:results} displays the results of baseline transfer-learning on pretrained Transformer-based models, models based on semantic similarity, and \textsc{mina}.

\paragraph{Baseline models:}
The baseline models perform reasonably well on Twitter, but not on TQ$^{\plus}_{\textsc{yt}}$: BERT achieves higher precision but lower recall, while RoBERTa shows the opposite behavior. It is thus evident that models adept at propaganda detection have room for improvement in identifying whataboutism, especially given their tendency to assign the majority label \textsc{nw} to test instances. 
% This issue is particularly relevant for real-world data, where \textsc{nw} is likely to dominate as the majority class.

\paragraph{Semantic similarity models:}
When applying the standard transfer-learning technique (\textit{i.e.}, fine-tuning on task-specific labeled data) to language models known for their excellent performance on various semantic similarity tasks, we find no consistent improvement over baseline models. Once again, all models display a strong inclination to classify TQ$^{\plus}_{\textsc{yt}}$ test instances into the majority class, with SBERT exhibiting the poorest performance in this regard. SimCSE, trained on hard negative samples from NLI datasets, shows relatively better performance, while the simplest model in this category, C\&S, achieves the highest $F_1$ score of 0.45. On TQ$^{\plus}_{\textsc{tw}}$, the results are consistently better, with SBERT in particular showing remarkable improvement. But these models fail to notably surpass the baseline transfer-learning methods. % using BERT, RoBERTa, or RoBERTa$_{\text{\,Irony}}$.

\begin{figure}[!t]
\centering
\begin{subfigure}[t]{.49\linewidth}
\centering
\includegraphics[width=.98\linewidth, cframe=gray 2pt]{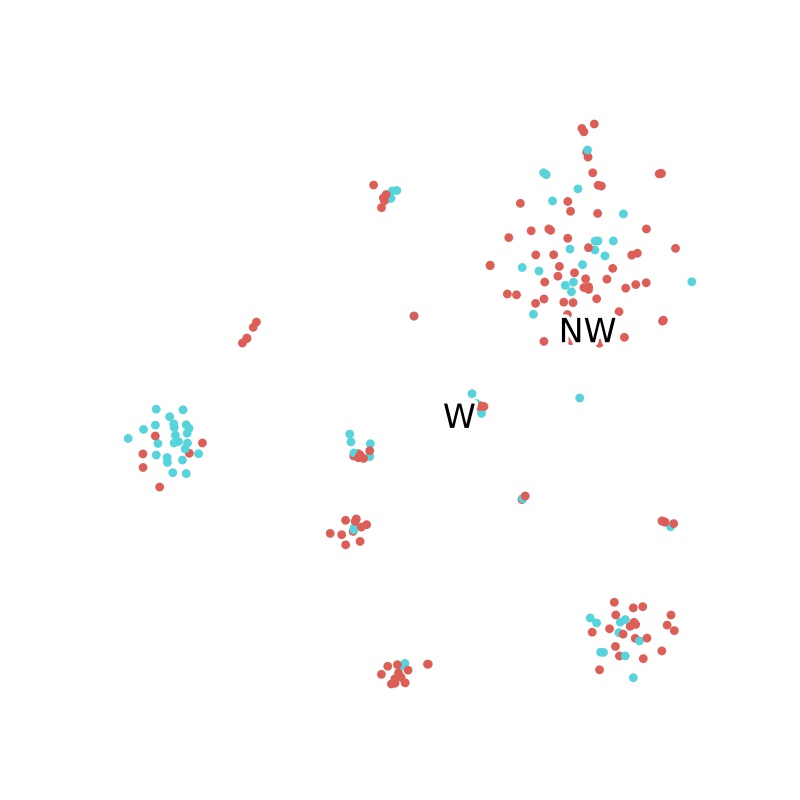}
\caption{Transfer-learning from pretrained language models.}
\label{fig:tsne1}
\end{subfigure}%
\hfill
\begin{subfigure}[t]{.49\linewidth}
\centering
\includegraphics[width=.98\linewidth, cframe=gray 2pt]{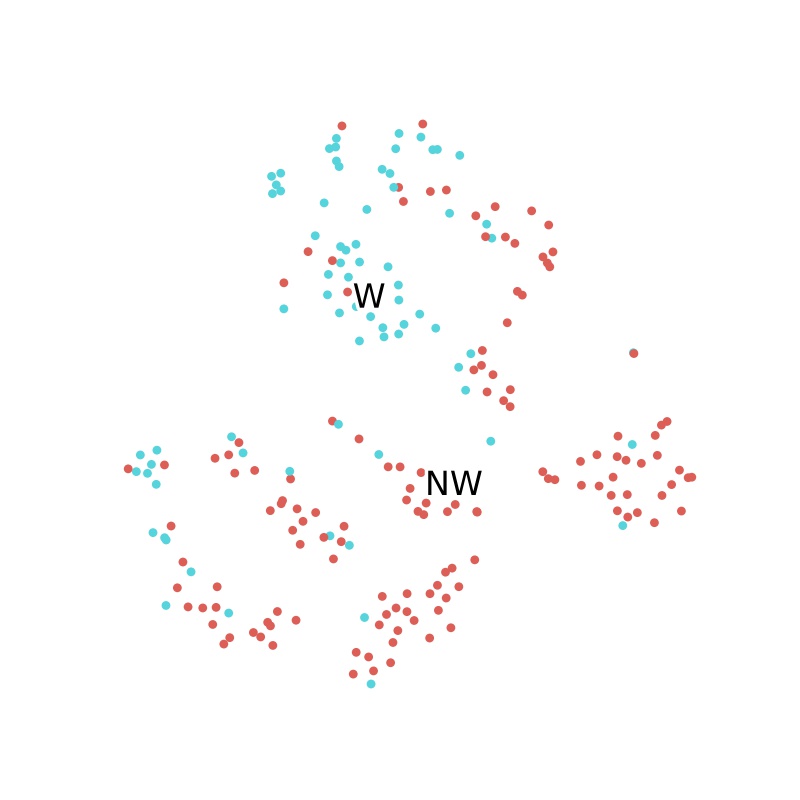}
\caption{Contrasting with MINA to capture pragmatic differences.}
\label{fig:tsne2}
\end{subfigure}
\caption{t-SNE visualization of the target class (\textsc{w}; blue), amid everything else (\textsc{nw}; red) with (a) transfer-learning with pretrained language models, and (b) mining negative samples using \textsc{mina}. The latter demonstrates better separability of the target class.}
\label{fig:tsne}
\end{figure}
\paragraph{Mining negatives with attention (\textsc{mina}):}
Even while employing contextual embeddings, the inability of the above models to effectively detect whataboutism suggests a limitation in capturing the kind of pragmatic nuances illustrated earlier (Fig.~\ref{fig:comment-examples}, \ref{fig:pragmatic-flexibility}). \textsc{mina}, however, is designed to leverage cross-attention scores across comments within a topic to model pragmatic context. To demonstrate its efficacy, we assess the \textit{worst}-performing models from previous categories, focusing on $F_1$ for the target class. On TQ$^{\plus}_{\textsc{yt}}$, $F_1$ of RoBERTa$_{\text{\,Irony}}$ increases by 22\%, rising from 0.3 to 0.52, while SBERT shows an exceptional improvement of 51\%, from 0.07 to 0.58: a 10\% boost over the next best result.\footnote{To compare, the SemEval-2020 task on propaganda detection~\cite{dasanmartino2020semeval} saw the best performance in detecting whataboutism, straw man arguments, and red herrings, achieve $F_1$ = 0.269. (these three closely-related fallacies were merged into a single group due to insufficient data.} When evaluated on TQ$^{\plus}_{\textsc{tw}}$, SBERT with \textsc{mina} outperforms other models, showing a 6\% improvement over vanilla SBERT and a 4\% improvement over RoBERTa and SimCSE. RoBERTa$_{\text{\,Irony}}$ initially exhibits high recall but low precision due to being predisposed to irony. However, \textsc{mina} often flips the output when irony is present in comments but not in other context tuple elements. The significant improvements we observe with \textsc{mina} are reflected in its effects on target class separability. Comparing the t-SNE \cite{hinton2002stochastic} visualizations in Fig.~\ref{fig:tsne1} and Fig.~\ref{fig:tsne2}, it becomes clear that while fine-tuning on a moderate amount of task-specific training data has limited utility, contrasting with \textsc{mina} improves the ability to distinguish whataboutism from everything else.
%Evaluating the same models on TQ$^{\plus}_{\textsc{tw}}$, we observe SBERT achieve the best $F_1$ on the target class -- a 6\% improvement over vanilla SBERT (from 0.69 to 0.75) and a 4\% improvement over the next best result (achieved by RoBERTa and SimCSE). A seemingly anomalous behavior is noted about RoBERTA$_{\text{\,Irony}}$. In the baseline setting, it remains predisposed to irony, which includes most \textsc{w} instances. This leads to its high recall but low precision. Upon employing MINA, this is reversed because the final classification layers essentially learn to flip the output of instances where the comment employs irony, but is coupled with unironic statements in its context tuple.

% baseline R_irony => by capturing all irony, you capture most W (but then, of course, poor precision)
% MINA R_irony => 

%\footnote{Appendix~\ref{app:mining-comparison-study} provides a supplemental study comparing MINA to other common mining strategies.}

\begingroup
\begin{table}[!t]
\setlength\tabcolsep{2pt}
\small
\centering
\begin{tabularx}{\linewidth}{@{}>{\raggedleft\arraybackslash}X >{\raggedleft\arraybackslash}X >{\raggedleft\arraybackslash}X}
\toprule
\multicolumn{3}{@{}p{\linewidth}}{(a) \textit{Varying $h$, with $d=4, c=1$} ($\sigma^2 = 0.0086$):}\\
$h$ (attention heads) & $\overline{F_1}$ (YouTube) & $\overline{F_1}$ (Twitter)\smallskip\\
32  & 0.52 & 0.71\\
64  & 0.52 & 0.70\\
128 & 0.51 & 0.70\\
384 & 0.48 & 0.69\smallskip\\
\end{tabularx}

\begin{tabularx}{\linewidth}{>{\raggedleft\arraybackslash}X >{\raggedleft\arraybackslash}X >{\raggedleft\arraybackslash}X}
\multicolumn{3}{@{}p{\linewidth}}{(b) \textit{Varying $c$, with $d=4, h=32$} ($\sigma^2 = 0.0072$):}\\
$c$ (context size) & $\overline{F_1}$ (YouTube) & $\overline{F_1}$ (Twitter)\smallskip\\
1 & 0.51 & 0.71\\
2 & 0.50 & 0.70\\
3 & 0.51 & 0.67\\
4 & 0.49 & 0.69\smallskip\\
\end{tabularx}

\begin{tabularx}{\linewidth}{>{\raggedleft\arraybackslash}X >{\raggedleft\arraybackslash}X >{\raggedleft\arraybackslash}X}
\multicolumn{3}{@{}p{\linewidth}}{(c) \textit{Varying $m$, with $d=4, h=32, c=1$} ($\sigma^2 = 0.0089$):}\\
$m$ (language model) & $\overline{F_1}$ (YouTube) & $\overline{F_1}$ (Twitter)\smallskip\\
\textsc{sbert} & 0.52 & 0.71\\
\textsc{r}o\textsc{bert}a\textsubscript{\,Irony} & 0.51 & 0.70\\
\bottomrule
\end{tabularx}
\caption{Ablation study on \textsc{mina} hyperparameters: $d$ (encoder layers), $h$ (attention heads), $c$ (context size), and $m$ (input embeddings). Mean $F_1$ and its variance $\sigma^2$ are reported over 50 runs.}
\label{tab:ablations}
\end{table}
\endgroup
\paragraph{Ablation experiments:}
To determine the optimal configuration for \textsc{mina}, we conduct a series of experiments varying hyperparameters $d$ (number of encoder layers), $h$ (number of attention heads), $d$ (number of comments for context size), and $m$ (language model for input embeddings) (see Table~\ref{tab:ablations}). The models are trained end-to-end with the configuration described earlier (\S\hspace{1pt}\ref{sec:experiments}). Each experiment varies one hyperparameter while keeping others constant. Due to random selection of comments in context tuples, each ablation undergoes 50 runs, with mean $F_1$ and variance reported. The best hyperparameter configuration and its best run informs the results in Table~\ref{tab:results}.\footnote{For reproducibility, context tuples and model weights from the best run are made available with our code and data.}

The results indicate no benefit in increasing the encoder layers, attention heads, and context size beyond 4, 32, and 2, respectively. For the model $m$, we use SBERT and RoBERTa$_\text{\,Irony}$. Across 50 trials, SBERT yields marginally higher average $F_1$ scores for YouTube (0.52, against 0.51 for RoBERTa$_\text{\,Irony}$) and Twitter (0.71, against 0.70 for RoBERTa$_\text{\,Irony}$). The variance across all 50 trials for selecting the embeddings was very low, at $\sigma^2 = 0.0089$.

\begingroup
\begin{table}[!t]
\setlength\tabcolsep{8pt}
{\small
\begin{tabularx}{\linewidth}{>{\raggedright\arraybackslash}X >{\centering\arraybackslash}X >{\centering\arraybackslash}X}
%\toprule
Method & $\overline{F_1}$ (YouTube) & $\overline{F_1}$ (Twitter) \\
\noalign{\vskip 0.8ex}
%\hdashline\noalign{\vskip 0.8ex}
MINA & \textbf{0.52} & \textbf{0.71} \\
Random & 0.50 & 0.67 \\
Cosine-Sim & 0.50 & 0.67\\
\bottomrule
\end{tabularx}
}
\caption{\textsc{mina}'s novel approach using cross-attention weights to construct pragmatically dissimilar contexts is a superior sample mining strategy.}\vspace{-11pt}
\label{tab:mining-comparison}
\end{table}
\endgroup
\paragraph{A comparison of mining strategies:}
We compare \textsc{mina}'s negative sample mining with two common strategies in recent contrastive mining literature: cosine similarity-based mining (\textit{e.g.}, \citealp{wang2021dense}) and random sampling \cite{jiang2021improving, xu2022negative}. Table~\ref{tab:mining-comparison} shows \textsc{mina}'s superior result, while the use of cosine measure is akin to random sampling for tuple construction. This further supports our conjecture that conventional measures of semantic similarity or contrast are inadequate for capturing pragmatic differences.

\paragraph{Challenges, negative results, and insights:}
In exploring several models and techniques to detect whataboutism, the path to \textsc{mina}, which uses negative sampling inspired by cross-attention, was marked by several unexpected outcomes. First, we experimented with the advanced pre-trained encoder DeBERTa \cite{he2021deberta} on the collection of YouTube comments, but its performance was unexpectedly inferior to BERT, achieving an $F_1$ score of 0.455 compared to BERT's 0.48.

Although \textit{modern large language models} (LLMs) like Llama-2-7b, Llama-2-13b \cite{touvron2023Llama}, and GPT-3.5-turbo \cite{brown2020gpt} have shown impressive capabilities in various tasks, they did not yield satisfactory results in detecting whataboutism. The limitations we observe align with those noted by \citet{ruiz2023detecting} in detecting fallacious argumentation. On Twitter comments, for instance, the highest $F_1$ score was 0.65, 4\% lower than the BERT baseline (Table~\ref{tab:results}). Even with sophisticated contextual prompting, no significant improvements were observed. GPT-3.5-turbo excelled with canonical whataboutism that closely mirrored syntactic patterns found in popular media, but struggled with nuanced, rephrased instances. Consequently, these LLMs were excluded from subsequent stages of experimentation in this work.

Our experiments included \textit{incorporating video transcripts as contextual information}, for which we leveraged Longformer \cite{beltagy2020longformer}. Contrary to initial expectations, these yielded lower precision and recall across all models. Analyzing the errors revealed that comments from other users are a better predictor than the contextual information derived from the video. This observation, plus our thematic analysis (\S\hspace{1pt}\ref{ssec:thematic-scrutiny}) and the success of \textsc{mina} strongly indicate that while the video remains important in setting the topic, detecting whataboutism in a comment often depends more on the broader social discourse surrounding the topic.

Lastly, we decided against \textit{batch-contrastive learning} due to two potential concerns raised by ~\citet{khosla2020contrastive} and \citet{qu2021coda}:
\begin{itemize*}
\item[(i)] large batch sizes are vital for it to be effective, and
\item[(ii)] it typically involves augmentation, which introduces noise that hinders pragmatic understanding.
\end{itemize*}
Further, it would still use cosine similarity on its own, which we have questioned in \S\hspace{1pt}\ref{ssec:experiments-semantic-similarity}.

\section{Related Work}
% Beyond the theoretical work discussed earlier (\S\ref{sec:pragmatic-flexibility}), our work is informed by
% \begin{itemize*}
% \item[(a)] empirical approaches to propaganda detection, and
% \item[(b)] modeling semantic textual similarity with contrastive learning.
% \end{itemize*}

%\noindent\textit{Propaganda detection.} 
\paragraph{Propaganda detection:}
Early propaganda detection~\cite{barroncadeno2019proppy} categorized sources based on external inputs, but concerns were raised due to noisy data and questions about epistemic bias~\cite{uscinski2013epistemology, marietta2015fact}. Traditionally, whataboutism was viewed as the tu quoque fallacy~\cite{fischer2021}. Another body of work, however, views it in the broader pragmatic and dialectic framework~\cite{aikin2008}. While this has gained prominence in theoretical studies~\cite{bowell2023}, computational approaches remain confined to viewing whataboutism purely as propaganda~\cite{dasanmartino2019fine, sahai2021breaking, baleato2023paper}. Our work connects to the pragma-dialectic view of whataboutism: overlapping with propaganda but not subsumed by it.

NLP tasks have traditionally favored utilizing pretrained models like BERT~\cite{devlin2019bert} and RoBERTa~\cite{liu2019roberta}, fine-tuning them for specific tasks. \citet{dasanmartino2019fine} accordingly approached propaganda detection, creating a corpus of news articles later used in shared tasks~\cite{dasanmartino2019findings, dasanmartino2020semeval}, where these pretrained models achieved top results~\cite{yoosuf2019fine, morio2020hitachi}. %451 news articles
While related, our research differs substantially. They target propaganda detection, treating whataboutism only as propagandist language while the limited dataset size precludes deeper analysis. %Plus, their use of news articles contrasts with social media.
\citet{dasanmartino2020semeval} nevertheless note the difficulty in identifying whataboutism. %, which further motivates our study.

Since identifying the precise textual span of propaganda is laborious and error-prone, we look to \citet{sahai2021breaking}, who label informal fallacies in social media comments. However, they continue to treat whataboutism purely as propaganda, while also turning to transfer learning through fine-tuning the models (\textit{i.e.}, the approach we use as baselines).

\paragraph{Framing and persuasion:} Unlike earlier studies, \citet{piskorski2023semeval} view whataboutism as strategic framing and persuasion rather than outright propaganda, thus offering a taxonomy that aligns more closely with our study. However, instances of whataboutism are sparse in their multilingual corpus (only 25 in English), limiting its utility in our work. Additionally, fine-tuned Transformer models or ensembles, as utilized by top-performing participants, yielded remarkably low $F_1$ scores for whataboutism and other related categories % such as `appeal to hypocrisy', `straw man', and `red herring'
(see \citealp{semeval2023sheffield}). The results of this shared task are noteworthy, yet they are confounded by ambiguity~\cite{semeval2023apatt, semeval2023nlubot} and low inter-annotator agreement. Addressing these challenges is crucial for an accurate interpretation or analysis of the findings.
% (221 across 9 languages; and only 25 in English)

%\paragraph{Framing and persuasion.} Many instances in our data may be viewed as attempts at re-framing instead of outright propaganda. In this sense, studies on framing and the language of persuasion are closely related to this work. Unlike earlier work, a recent SemEval task
%% (https://aclanthology.org/2023.semeval-1.317v2.pdf)
%viewed such use of language as strategic framing, and categorized whataboutism as a persuasion technique instead of propaganda. Thus, their taxonomy aligns more with our work than previous studies on propaganda detection. Their corpus, however has relatively fewer instances of whataboutism even when aggregating across all the languages in their multilingual setting (221 across 9 languages; and only 25 in English).
%Even the top results among participating teams use fine-tuned transformer models (or ensembles of multiple such models), and report very low F1 scores (both micro and macro averaged). In fact, some of the top performers report 0.0 F1 on not just whataboutism, but also on related categories like 'appeal to hypocrisy', 'straw man', and 'red herring'.
%% SheffieldVeraAI (https://aclanthology.org/2023.semeval-1.275)
%Some top participants have pointed out ambiguous annotations, which is understandable given the low inter-annotator agreement reported for the dataset used for this shared task.
%% APatt (https://aclanthology.org/2023.semeval-1.51.pdf)

%\noindent\textit{Modeling (dis)\,similarity.}
\paragraph{Modeling (dis)\,similarity.}
Whataboutism, a discursive maneuver for diversion and centering, relates to semantic textual similarity (STS)~\cite{cer2017semeval}. Modeled as a graph, STS becomes a node classification task using graph neural networks (GNN). \citet{huang2021combining}, however, show that minor modifications to classical graph-based learning algorithms (\textit{e.g.},~\citealp{zhou2003learning}) outperform large GNNs despite far fewer parameters and less training time. Their approach (C\&S) does not require careful early stopping criteria or a large validation set, making it another competitive baseline.

STS models do not capture whataboutism's pragmatic complexity. We address this by sharing a core intuition with contrastive learning: bring similar instances closer while pushing apart dissimilar ones~\cite{weinberger2009distance}. Augmentation methods used in contrastive learning, however, are ill-suited for pragmatics. %\footnote{For instance, augmentation in NLP tasks often involves back-translation, word reordering, etc.~\cite{wu2020clear}.}
Instead, we opt for SimCSE~\cite{gao2021simcse}. While hard negative mining has proven effective~\cite{xiong2021approximate}, the \textit{use of attention weights to model pragmatic differences} is a novel contribution of this work. Here, notable work includes~\citet{yuquan2018acv}, who integrate attention weights with syntax, and \citet{yamagiwa2022improving}, who fuse self-attention matrices with word-mover’s distance to obtain promising STS results.

\section{Conclusion}
Our study underscores areas where state-of-the-art propaganda detection models often fail to distinguish propagandist use of whataboutism from valid argumentation tools or figurative language use like irony, revealing how their suitability for this task may be improved. Addressing this within the pragma-dialectic framework, we illustrate the need to develop models capable of understanding pragmatic variations. We find traditional similarity measures ineffective for this purpose. Cross-attention proves valuable instead. We thus propose a novel methodology, \textsc{mina}, for mining negative samples based on cross-attention, which achieves superior results in whataboutism detection and suggests its utility in other tasks that require grounding in ambient social discourse. This study is based on our contribution of two annotated datasets that can facilitate further research in related areas.

\section{Limitations}
% \begin{enumerate}[leftmargin=*,nolistsep, noitemsep]
% \item
Emotive topics may influence the labels assigned to specific instances. We report Fleiss' kappa, adjusting for the possibility of chance agreement, but our approach does not take into account individual viewpoints. We thus advocate imbibing this into \textit{perspectivist} research~\cite{cabitza2023toward}.
% \item

We have shared model weights, code, and datasets for replicability. \textsc{mina}'s reliance on randomized tuple batches, however, may cause slight changes in retrained models in spite of extremely low variance (reported in Table~\ref{tab:ablations}).
% \item 

The insights regarding the underwhelming performance of modern LLMs are crucial for understanding their limitations in handling tasks like whataboutism detection, where pragmatic nuances are important. In this regard, a deeper exploration of in-context learning is warranted.
% \item 

Finally, we note that the TQ$^{\plus}$ collections are imbalanced, with a minority target class. This, however, mirrors real-world scenarios and may not be a disadvantage after all.

\bibliography{references}

\clearpage
\setcounter{footnote}{0}
\appendix
\begin{figure*}[!htp]
\begin{subfigure}[h]{0.49\textwidth}
\includegraphics[trim=5.25cm 1cm 11.45cm 11cm, clip=true, width=\linewidth, height=6.4cm]{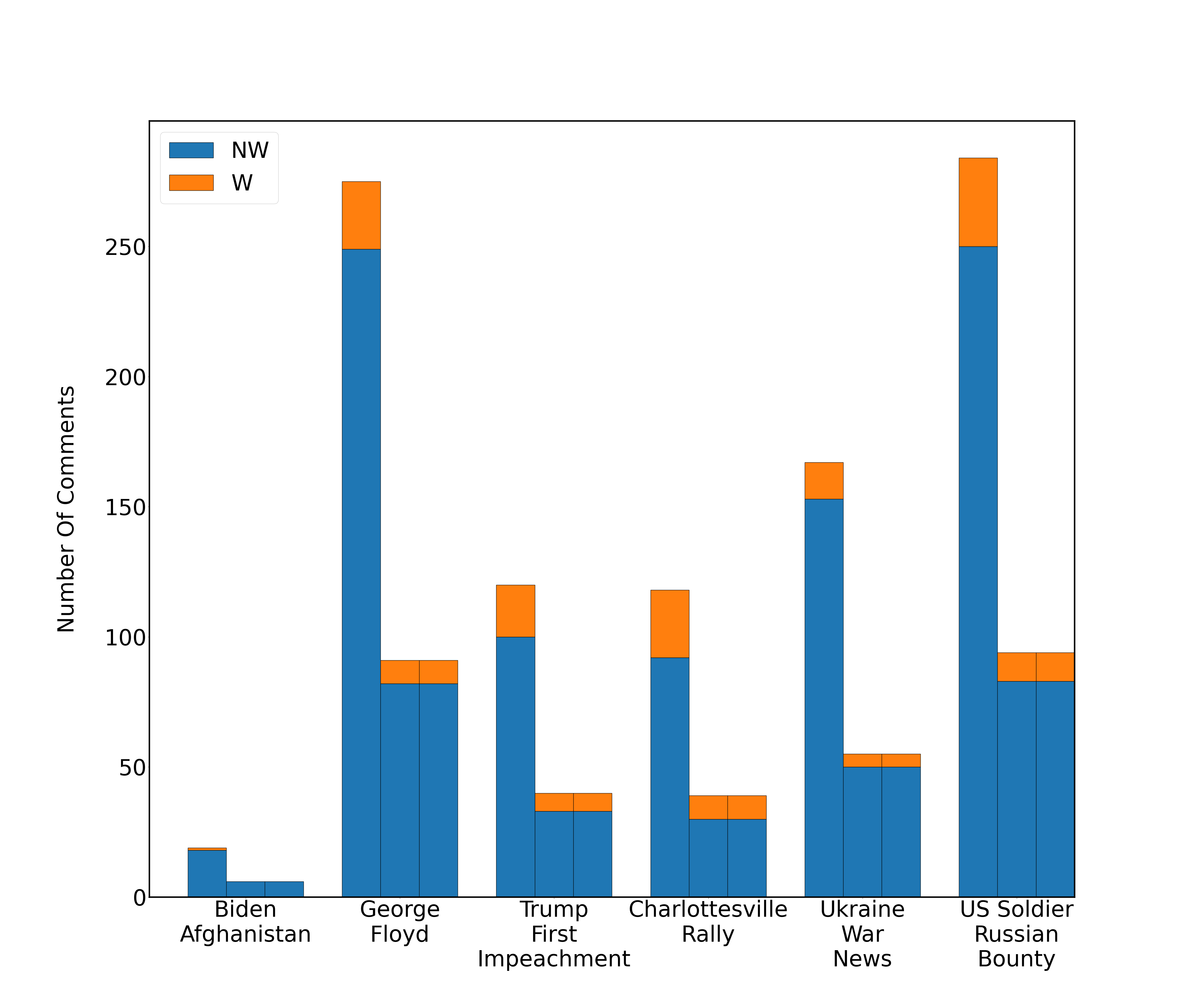}
\caption{TQ$^{\plus}_{\textsc{yt}}$}
\label{fig:dataset-statistics-yt}
\end{subfigure}%
\hfill
\begin{subfigure}[h]{0.49\textwidth}
\includegraphics[trim=5.25cm 1cm 11.45cm 11cm, clip=true, width=\linewidth, height=6.4cm]{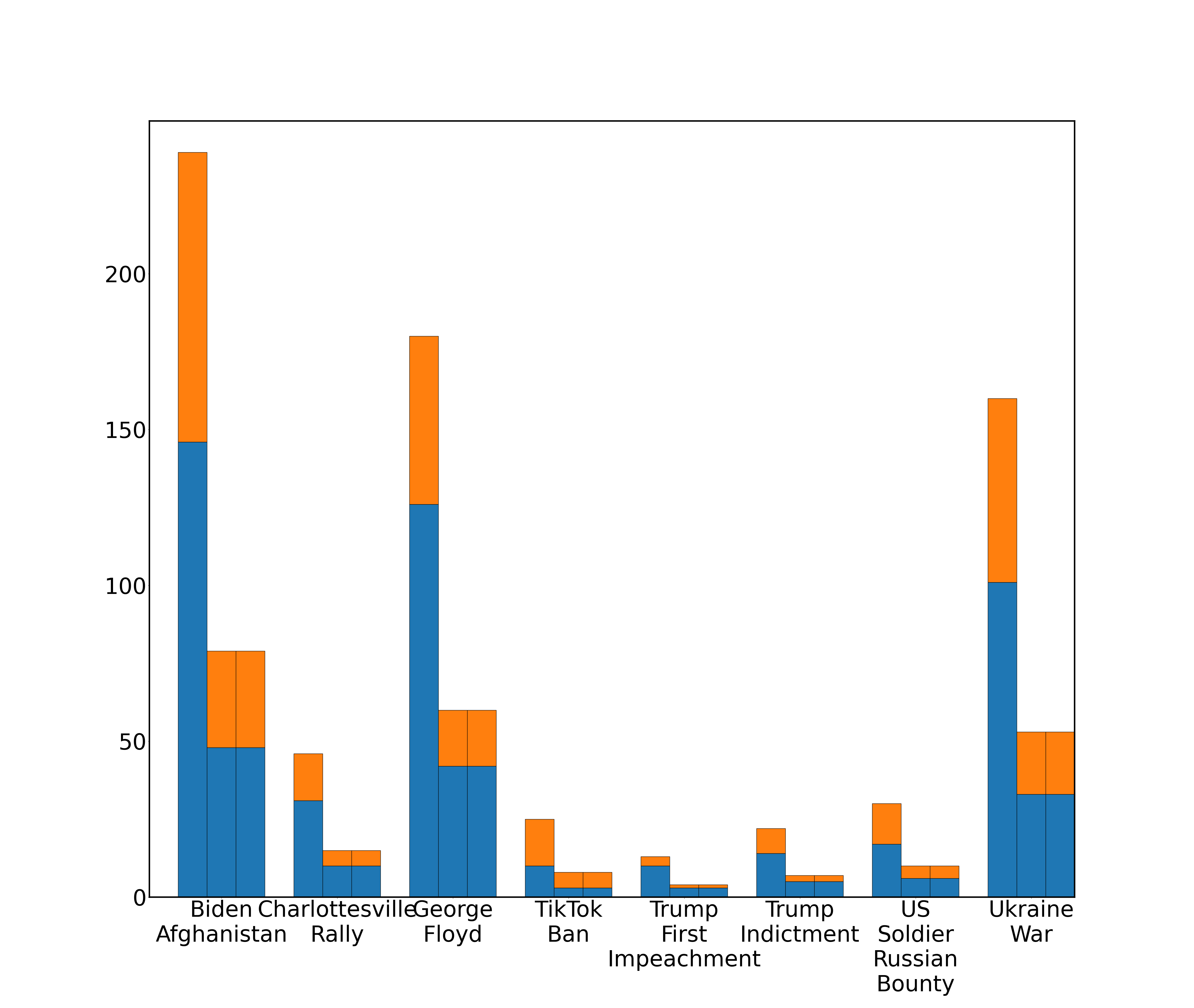}
\caption{TQ$^{\plus}_{\textsc{tw}}$}
\label{fig:dataset-statistics-twitter}
\end{subfigure}%
\caption{Labeled data distribution, spanning six divisive sociopolitical topics for the collection of YouTube comments, and eight such topics for the collection of Twitter responses. Each topic is partitioned into three sets (training, validation, and test), each with comments labeled as whataboutism (\textsc{w}) or not (\textsc{nw}). %For cross-topic detection, the three topics on the left (blue) are used to train and develop models, while the others (green) are used for testing.
}
\label{fig:dataset-statistics}
\end{figure*}

% Maybe this is better (have a "Data Statement" instead of a "Datasheet"):
% Emily Bender and Batya Friedman. Data Statements for Natural Language Processing: Toward Mitigating System Bias and Enabling Better Science

\section{Datasheet}
\label{appendix:datasheet}
This datasheet is included based on the recommendations of \citet{gebru2018datasheets}. It serves to document the creation, composition, intended uses, and maintenance of the TQ\textsuperscript{+} dataset (``tu quoque and beyond'') released with this work. We hope this will facilitate its better usage, and further encourage transparency, accountability, and reproducibility.\footnote{For the sake of brevity, we have not included information that is included in the main body of this paper.}

\paragraph{Why was the dataset created?}
TQ\textsuperscript{+} was created to enable research on identifying expressions of whataboutism in social media posts made by active participating users: given an English (en-US) comment made in the specific context, identify whether the comment expresses whataboutism. Intentionally created for this task, the dataset comprises of several hot-button sociopolitical topics, where whataboutism is more likely to be employed. While there exist a few datasets for propaganda detection~\cite{dasanmartino2019findings}, there is no corpus exclusively devoted to the study of whataboutism, which has unique sociopolitical and linguistic properties (ranging from semantics to pragmatics) that are distinct from other propagandist maneuvers. Furthermore, many who employ whataboutism are neither the creators nor the intentional disseminators of propaganda. Instead, many instances are attempts made by users to challenge or re-center a narrative (better captured in the taxonomic changes seen in \citet{piskorski2023semeval}). These factors distinguish TQ\textsuperscript{+} from earlier datasets.

\paragraph{What other tasks could the dataset be used for?}
TQ\textsuperscript{+} can be used for various types of modeling or comprehension of whataboutism in social media commentary. For instance, one may study syntactic patterns and their correlation with expressions of whataboutism, or how other users react to comments that express whataboutism. It can also be used to learn from adversarial settings, wherein this dataset can be used to automatically generate whataboutism on politically divisive topics. Furthermore, we expect TQ\textsuperscript{+} to be useful for discourse analyses and other studies conducted by social scientists and media communication researchers.

\paragraph{Has the dataset been used for any tasks already?}
This is a novel dataset. As of June 2023, it has not been used for any other task or publication.

\subsection{Dataset composition}
TQ\textsuperscript{+} comprises two sub-datasets: TQ$^{\plus}_{\textsc{yt}}$, which consists of comments made by active users on YouTube videos, and TQ$^{\plus}_{\textsc{tw}}$, which consists of Twitter posts.

Fig.~\ref{fig:dataset-statistics-yt} shows the distribution of these comments and their labels across the topics in the YouTube corpus.The Twitter corpus has 6 topics that are in TQ$^{\plus}_{\textsc{yt}}$, and additionally, \textit{Tiktok ban} and \textit{Trump indictment}. The data distribution across all these topics, along with the distribution of the binary ground-truth labels, is shown in Fig.~\ref{fig:dataset-statistics-twitter}.

\paragraph{What are the instances?}
\begin{itemize*}
\item[(1)] TQ$^{\plus}_{\textsc{yt}}$: Each instance is a comment written by a YouTube user. Each comment is provided along with a link to the corresponding YouTube video and its title, and labeled as \textit{whataboutism} or not (1/0).
\item[(2)] TQ$^{\plus}_{\textsc{tw}}$: Each instance consists of a tweet written by a Twitter user and a comment written by a different Twitter user in response to that tweet. A link to the corresponding tweets is provided. Each comment is labeled as \textit{whataboutism} or not (1/0).
\end{itemize*}

For both collections, three annotators worked independently to generate the labels. Since multiple annotator judgments are obtained before the majority vote, we take cognizance of the recommendations made by~\citet{prabhakaran2021releasing}, and provide the individual annotator labels in the released datasets for flexible future use.

\paragraph{Are relationships between instances made explicit in the data?}~\\
\noindent\textit{TQ\textsuperscript{+}YouTube}: There are no explicit relationships between any two comments, other than their correspondence with the same YouTube video.

\noindent\textit{TQ\textsuperscript{+}Twitter}: 
There is no direct connection between any two comments, except for the fact that they both correspond to the same tweet.

\paragraph{How many instances of each type are there?}~\\
\noindent\textit{TQ\textsuperscript{+}YouTube}: There are $1,642$ instances in total, with 202 comments labeled as whataboutism (after aggregation via majority voting).

\noindent\textit{TQ\textsuperscript{+}Twitter}: 
 There are $1,202$ instances in total, with 508 comments labeled as whataboutism.

\paragraph{Is everything included, or does the data rely on external resources?} Everything is included. Due to the dynamic nature of social media and potential policy changes, however, future use of the data collection script may not gather the same data even if the comments remain on the platform.

\paragraph{Are there recommended data splits or evaluation measures?} The dataset comes with specified training, validation, and test splits ($80$\%, $5$\%, and $15$\%, respectively). Due to class imbalance in the naturally occurring label distribution, the recommended measures are macro average precision, recall, and $F_1$ scores.

\subsection{Data preprocessing}
Comments in any language other than English (en-US) were discarded. Non-linguistic characters such as emojis were removed from each instance. 

\subsection{Dataset distribution and maintenance}
The dataset is distributed together with the code, under the MIT license.
%under the most permissive public dedication license, CC0, so that others may reuse, distribute, modify, adapt, and build upon this dataset in any medium or format. It is released along with this publication, and available on the Dryad Digital Repository.
There are no fees or access/export restrictions on this dataset.
% see this for data repo: https://library.stonybrook.edu/2023/03/20/stony-brook-university-joins-the-dryad-digital-repository/
% Dryad requires CC0 license for the dataset. If we want something else, like CC-BY-SA, we might want to use SBU's Academic Commons. See here: https://guides.library.stonybrook.edu/research-data/access-sharing-reuse

\subsection{Legal \& ethical considerations}
\paragraph{If the dataset relates to people, or was generated by people, were they informed about the data collection?} The data was collected from public web sources, through publicly available APIs. The authors of the comments collected in the dataset are presumably aware that their posts are public, but there was no mechanism of informing them explicitly about the development of the dataset.

\paragraph{Ethical review applications or approvals:} N/A.

\paragraph{Were there any provisions of privacy guarantees?} No. Neither the dataset nor the data collection process includes any names or user handles.

\paragraph{Does the dataset comply with the EU General Data Protection Regulation (GDPR)?} TQ\textsuperscript{+} does not contain any personal data of the authors of the collected comments. Further, it does not contain sensitive or confidential information.

\paragraph{Does the dataset contain information that might be considered inappropriate or offensive?} Some YouTube comments and Twitter posts may contain inappropriate or offensive language, but their presence is negligible. For example, in a random sampling of $273$ YouTube comments ($7.8$\% of the complete dataset), the authors of this publication found no such instance.
\newtheorem{theorem}{Theorem}
\newtheorem{definition}{Definition}

\begin{figure*}[!t]
\begin{minipage}[t]{0.24\textwidth}
\includegraphics[width=\linewidth, height=0.8\linewidth]{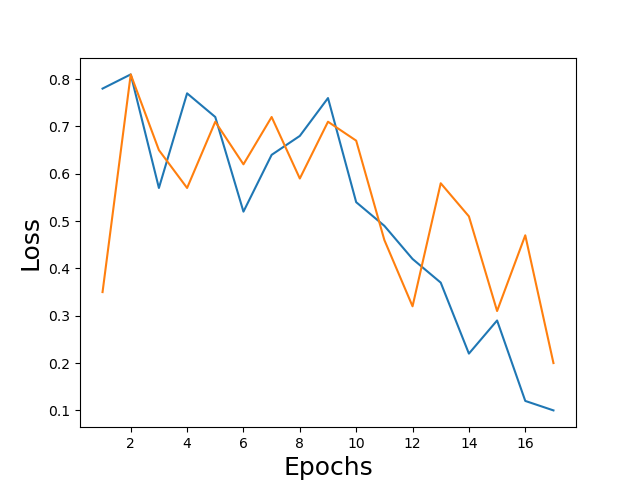}
%\caption{\textsc{sbert} with \textsc{mina} on TQ$^{\plus}_{\textsc{yt}}$}
\label{fig:sbert-mina-yt}
\end{minipage}%
\hfill
\begin{minipage}[t]{0.24\textwidth}
\includegraphics[width=\linewidth, height=0.8\linewidth]{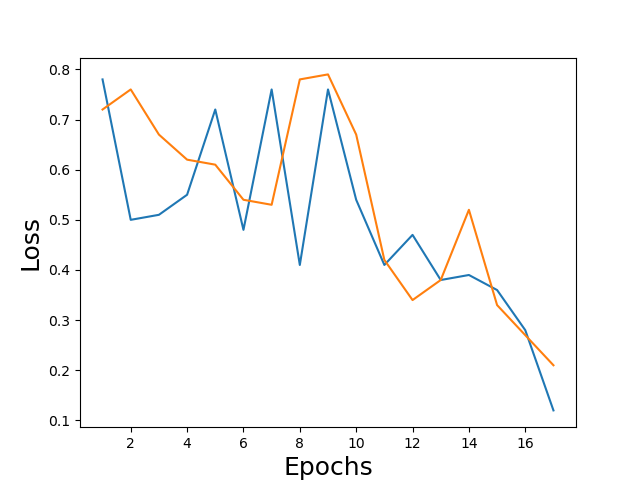}
%\caption{\textsc{r}o\textsc{bert}a\textsubscript{\,Irony} with \textsc{mina} on TQ$^{\plus}_{\textsc{yt}}$}
\label{fig:roberta-irony-mina-yt}
\end{minipage}%
\hfill
\begin{minipage}[t]{0.24\textwidth}
\includegraphics[width=\linewidth, height=0.8\linewidth]{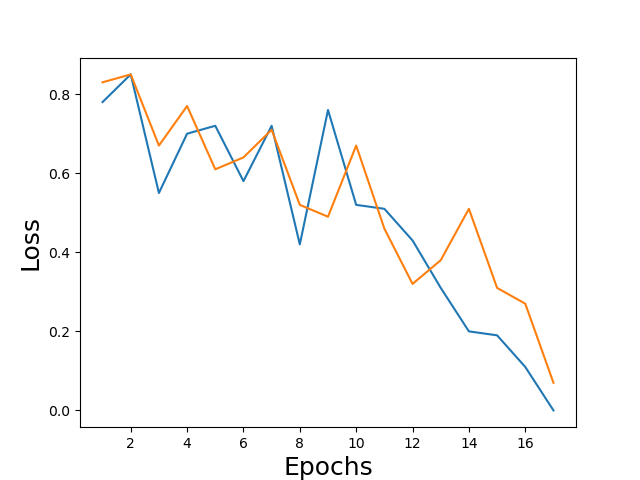}
%\caption{\textsc{sbert} with \textsc{mina} on TQ$^{\plus}_{\textsc{tw}}$}
\label{fig:sbert-mina-tw}
\end{minipage}%
\hfill
\begin{minipage}[t]{0.24\textwidth}
\includegraphics[width=\linewidth, height=0.8\linewidth]{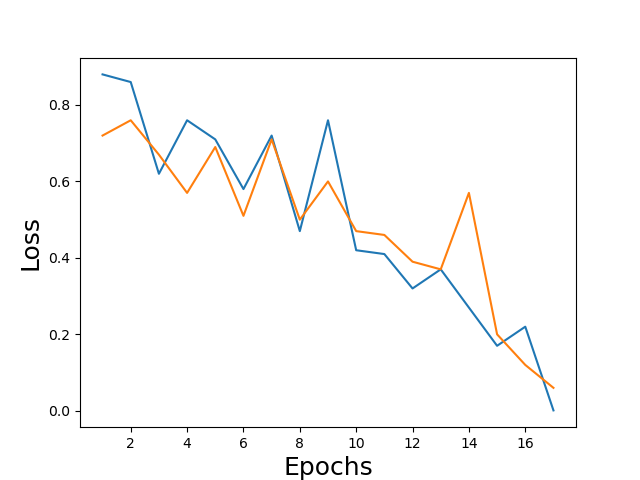}
%\caption{\textsc{r}o\textsc{bert}a\textsubscript{\,Irony} with \textsc{mina} on TQ$^{\plus}_{\textsc{tw}}$}
\label{fig:roberta-irony-mina-tw}
\end{minipage}
\vspace*{-11pt}
\caption{Training (\sqbox{bestblue}) and validation (\sqbox{orange}) losses over all epochs for the \textsc{mina}-enhanced models reported in Table~\ref{tab:results}. From left to right: SBERT and RoBERTa\textsubscript{\,Irony} on YouTube comments; and then, SBERT and RoBERTa\textsubscript{\,Irony} on Twitter comments. Consistent trends between training and validation losses, where both decrease without significant divergence, indicate that a model is effectively generalizing to unseen data and not overfitting to the training set.}
\label{fig:loss-over-epochs}
\end{figure*}

% (1) https://www.cs.cornell.edu/jeh/book.pdf (see page 21)
% (2) https://courses.cs.washington.edu/courses/cse521/16sp/521-lecture-6.pdf (see Section 6.3)
% combine the rotational symmetry from (2) with this proof: https://math.stackexchange.com/a/995680
\section{Proof of near-orthogonality}
\label{app:proof}
Formal proof of the statement (\S\hspace{1pt}\ref{ssec:experiments-semantic-similarity}): Cosine similarity may present challenges in high-dimensional vector spaces, as it becomes highly probable that any two vectors are nearly orthogonal.
% \begin{quote}
% ``Cosine similarity may present challenges in high-dimensional vector spaces, as it becomes highly probable that any two vectors are nearly orthogonal.''
% \end{quote}
\begin{theorem}[Bernstein's Inequality]
\label{thm:chernoff-bound}
Let $X_1, \ldots, X_n$ be independent random variables with $\mathbb{E}\left[X_i\right] = 0$ and $\vert X_i \vert \leq 1$ for all $i$. Let $X = \sum_{i \in [n]} X_i$, and let $\sigma^2$ denote the variance of $X_i$. Then,
\begin{equation*}
\mathbb{P}\left[ \vert X \vert \geq k\sigma \right] \leq 2e^{-k^2/4n}
\end{equation*}
\end{theorem}

\begin{definition}
A $d$-dimensional \textbf{ball} is defined as $B_d = \left\lbrace
\left(x_1, \ldots, x_d\right) : \lVert\mathbf{x}\rVert_2 \leq 1
\right\rbrace$.
\end{definition}

\begin{definition}
A vector $\mathbf{g} \in \mathbb{R}^d$ is a \textbf{Gaussian vector} if each $g_i\,(1 \leq i \leq d)$ is a uniformly and independently chosen $\mathcal{N}(0,1)$ random variable.
\end{definition}

\begin{theorem}
\label{thm:main}
Let $\mathbf{v}$ be a unit vector in $\mathbb{R}^d$, and let $\mathbf{x} = (x_1, \ldots, x_d)$ be a Gaussian vector on the surface of $B_d$ obtained by choosing each $x_i \in \{\pm 1\}$ and then normalizing via multiplication by $1/\sqrt{n}$. Further, let $X$ denote the random variable $\mathbf{v}\cdot\mathbf{x} = \sum_{1 \leq i \leq d} v_i x_i$. Then,
\begin{equation*}
\mathbb{P}\left( \vert X \vert \geq \epsilon \right) \leq 2e^{-d\epsilon^2/4}
\end{equation*}
\end{theorem}

% % there are many proofs. we are following the argument here:
% % https://www.cs.princeton.edu/courses/archive/fall13/cos521/lecnotes/lec11.pdf
\begin{proof}
\begin{align*}
\mu
  &= \mathbb{E}\left[(X)\right] = \mathbb{E}\left[\left( \sum_{i = 1}^d v_i x_i \right)\right] \\
  &= \sum_{1 \leq i \leq d} \mathbb{E}\left[ v_i x_i \right] = 0. \\
\sigma^2
  &= \mathbb{E}\left[ \left(\sum_{i = 1}^d v_i x_i \right)^2 \right] \\
  &= \mathbb{E}\left[ \sum_{1 \leq i \leq j \leq d} v_i v_j x_i x_j \right] \\
  &= \sum_{1 \leq i \leq d} v_i^2\mathbb{E}[x_i^2] + \sum_{\substack{i, j = 1 \\ i \neq j}}^d v_i v_j \mathbb{E}[x_i x_j].
\end{align*}
Since $\mathbb{E}(x_i^2) = 1/n$, and if $i \neq j$, $\mathbb{E}(x_i x_j) = 0$,
%\frac{1}{2}\cdot\frac{1}{n} + \frac{1}{2}\cdot\frac{-1}{n} = 0$.
\begin{align*}
\sigma^2
  &= \sum_{i = 1}^d \frac{v_i^2}{d} = \frac{1}{d}.
\end{align*}
By Bernstein's inequality (Theorem~\ref{thm:chernoff-bound}),
\begin{equation*}
\mathbb{P}\left(\vert X \vert \geq \epsilon\right) \leq 2e^{-\epsilon^2\cdot d/4}
\end{equation*}
\end{proof}
Thus, if two unit vectors $\mathbf{x}$ and $\mathbf{y}$ are chosen at random from $\mathbb{R}^d$, and the random variable $X$ denotes the inner product $\langle \mathbf{x}, \mathbf{y} \rangle$, the upper bound of Theorem~\ref{thm:main} shows that for high values of $d$, the inner product will have a very small value with high probability. In other words, in high-dimensional vector spaces, two randomly picked vectors are \textit{nearly orthogonal} with high probability.
% more details here (but we don't really need it):
% https://www.cs.unm.edu/~saia/classes/506-s20/lec/JLProjection.pdf

\section{Training and Validation Loss Analysis}
This appendix presents the learning curves for the models enhanced with \textsc{mina}, illustrating the training and validation loss over all epochs. These plots, shown in Fig.~\ref{fig:loss-over-epochs}, depict their performances and generalization capabilities, addressing potential concerns about overfitting due to the moderate size of the TQ$^{\plus}$ collections.

\end{document}